\documentclass[letterpaper]{article} 
\usepackage[]{aaai2026}  
\usepackage{times}  
\usepackage{helvet}  
\usepackage{courier}  
\usepackage[hyphens]{url}  
\usepackage{graphicx} 
\usepackage{natbib}  
\usepackage{caption} 
\usepackage[utf8]{inputenc} 
\usepackage[T1]{fontenc}    
\usepackage{booktabs}       
\usepackage{amsfonts}       
\usepackage{nicefrac}       
\usepackage{microtype}      
\usepackage{xcolor}         

\usepackage{subcaption}
\usepackage{graphicx}
\usepackage{booktabs}
\usepackage{algorithm}
\usepackage{algpseudocode}
\usepackage{latexsym}
\usepackage{amsmath,amssymb}
\usepackage{mathtools}
\usepackage{stmaryrd}
\usepackage{bm}
\usepackage{amsthm}
 
\providecommand{\argmin}{\operatornamewithlimits{argmin}}

\DeclareMathOperator{\diag}{diag}

\providecommand{\R}{\mathbb{R}} 
\providecommand{\E}{\mathbb{E}} 
\providecommand{\T}{\mathrm{T}} 

\renewcommand{\geq}{\geqslant}
\renewcommand{\leq}{\leqslant}
\DeclarePairedDelimiterX{\inner}[2]{\langle}{\rangle}{#1, #2}
\DeclarePairedDelimiter{\norm}{\lVert}{\rVert}
\DeclarePairedDelimiter{\abs}{\lvert}{\rvert}

\usepackage{cleveref}
\newtheorem{theorem}{Theorem}[]
\newtheorem{proposition}[theorem]{Proposition}

\newtheorem{lemma}[theorem]{Lemma}
\theoremstyle{definition}

\newtheorem{remark}[]{Remark}

\newtheorem{assumption}[]{Assumption}
\usepackage{ifthen}

\providecommand{\calD}{\mathcal{D}}
\providecommand{\piref}{\pi_{\mathrm{ref}}}

\providecommand{\rv}{\bm{r}}

\usepackage{newfloat}
\usepackage{listings}
\DeclareCaptionStyle{ruled}{labelfont=normalfont,labelsep=colon,strut=off} 
\floatstyle{ruled}
\newfloat{listing}{tb}{lst}{}
\floatname{listing}{Listing}
\title{Cost-Minimized Label-Flipping Poisoning Attack to LLM Alignment}

\author {
    Shigeki Kusaka\equalcontrib\textsuperscript{\rm 1},
    Keita Saito\equalcontrib\textsuperscript{\rm 1},
    Mikoto Kudo\equalcontrib\textsuperscript{\rm 1,2},
    Takumi Tanabe\textsuperscript{\rm 3},
    Akifumi Wachi\textsuperscript{\rm 3},\\
    Youhei Akimoto\textsuperscript{\rm 1,2,4}
}
\affiliations {
    \textsuperscript{\rm 1}University of Tsukuba,
    \textsuperscript{\rm 2}RIKEN AIP, 
    \textsuperscript{\rm 3}LY Corporation, 
    \textsuperscript{\rm 4}Institute of Science Tokyo\\
    \{shigeki.kusaka, keita.saito, mikoto\}@bbo.cs.tsukuba.ac.jp,\\
    \{akifumi.wachi, takumi.tanabe\}@lycorp.co.jp, \\
    akimoto@cs.tsukuba.ac.jp
}

\begin{document}

\maketitle

\begin{abstract}

Large language models (LLMs) are increasingly deployed in real-world systems, making it critical to understand their vulnerabilities. While data poisoning attacks during RLHF/DPO alignment have been studied empirically, their theoretical foundations remain unclear. We investigate the minimum-cost poisoning attack required to steer an LLM’s policy toward an attacker’s target by flipping preference labels during RLHF/DPO, without altering the compared outputs. We formulate this as a convex optimization problem with linear constraints, deriving lower and upper bounds on the minimum attack cost. As a byproduct of this theoretical analysis, we show that any existing label-flipping attack can be post-processed via our proposed method to reduce the number of label flips required while preserving the intended poisoning effect. Empirical results demonstrate that this cost-minimization post-processing can significantly reduce poisoning costs over baselines, particularly when the reward model’s feature dimension is small relative to the dataset size. These findings highlight fundamental vulnerabilities in RLHF/DPO pipelines and provide tools to evaluate their robustness against low-cost poisoning attacks.
\end{abstract}

\begin{links}
\link{Code}{https://github.com/akimotolab/PoisoningCostMinimization}
\end{links}

\section{Introduction}

\paragraph{Vulnerability of LLM}
As large language models (LLMs) are increasingly deployed in real-world applications,
understanding their vulnerabilities is essential for ensuring their effective and safe use.
Adversarial attacks expose these vulnerabilities, supporting red-teaming efforts \citep{shayegani2023surveyvulnerabilitieslargelanguage}. 
Representative adversarial attacks at inference time include jail-breaking \citep{wei2023neurips,chao2023jailbreaking,Mehrotra2024neurips,zou2023universal} and prompt injection \citep{greshake2023aisec,Liu2024usenix}, where attackers craft malicious inputs to elicit unintended outputs from LLMs. 
At training time, \emph{data poisoning attacks} modify the training dataset to induce undesired behaviors or embed backdoor triggers in the resulting LLM. In this paper, we focus on the data poisoning attack on LLMs. 

\paragraph{Poisoning attacks on preference alignment}

LLMs are susceptible to data poisoning attacks due to their multi-stage training pipeline, which typically includes pre-training, supervised fine-tuning (SFT), and alignment via reinforcement learning from human feedback (RLHF) \citep{rlhf} or direct preference optimization (DPO) \citep{dpo}. 
Recent empirical studies have demonstrated that LLMs can be compromised through poisoning during both the SFT phase \citep{wan2023poisoning,Shu2023neurips} and the RLHF/DPO phase \citep{wu2024preference,rlhfpoison,pathmanathan2024poisoning,baumgartner2024best,pathmanathan2024advbdgen,tramer2024universal}, raising concerns about their robustness to adversarial manipulation. 
However, the theoretical foundations of poisoning attacks in the RLHF/DPO phase remain largely unexplored, leaving open questions about the fundamental vulnerabilities of these methods and potential defenses.
A theoretical understanding is crucial to ascertain the worst-case scenarios for victims, which empirical studies cannot fully reveal.

\paragraph{Objective}

We theoretically investigate the minimum cost of the attack to successfully steer the optimal 
LLM policy toward the attacker's target policy during the RLHF/DPO phase. 
We consider that an attacker who operates as an annotator, tasked with evaluating two outputs $y$ and $z$ in a given context $x$, and providing a binary preference label ($w = 1$ if $y$ is preferred; otherwise $w=-1$). 
While the attacker can not modify $(x, y, z)$, they can arbitrarily set the preference label $w$.
Our goal is to determine the minimum number of label flips from the benign labels required to induce the attacker's desired behavior and to design a method for constructing a malicious dataset that achieves this objective with minimal cost. 
By quantifying these costs, our analysis is expected to guide the design of robust RLHF/DPO pipelines that can detect or mitigate such low-cost poisoning attacks, ensuring the safe deployment of LLMs in real-world 
settings.

\paragraph{Contributions}

In this work, we provide the first theoretical analysis of the minimal cost required to steer LLM policies via label-flipping attacks during RLHF/DPO alignment. By formulating the problem as a convex (or linear) optimization, we derive tight lower and upper bounds on the minimum number of label flips needed to induce a target policy. As a byproduct of this analysis, we develop a post-processing method that can be applied to any existing label-flipping attack to reduce its cost while preserving its intended poisoning effect. This approach is particularly effective in practical LLM alignment pipelines, where the dataset size is significantly greater than the feature dimension of the reward model, enabling attackers to exploit redundancy in the data in the feature space to minimize the cost of targeted poisoning. 

\section{Preliminaries}

LLM alignment is often conducted in two stages: supervised fine-tuning (SFT) and learning from human feedback. 
Given an LLM pre-trained with a large corpus in an unsupervised manner, it is fine-tuned using human-annotated input-output pairs $(x, y)$ to produce more relevant output for specific downstream tasks of interest. 
Learning from human feedback then aims to further align LLMs with human preferences. 
In this step, the LLM is trained using a dataset of preferences, $\calD_L = \{(x, y, z, w)\}$, where $x$ is the input, $y$ and $z$ are two candidate outputs, and $w \in \{-1, 1\}$ is the preference label indicating whether $y$ is preferred to $z$ ($w = 1$) or not ($w = -1$). The LLM is trained to assign higher probabilities to preferred outputs.
Reinforcement learning from human feedback (RLHF) is often employed for this purpose. 

RLHF first trains a reward model $r(x, y)$ from the preference dataset. 
The human preference is typically modeled as the Bradley-Terry model:
\begin{equation}
\Pr[w = 1 \mid x, y, z] = \sigma(r(x, y) - r(x, z)) ,
\end{equation}
where $\sigma(t) = \frac{1}{1 + \exp(-t)}$ is the sigmoid function.
The reward model is trained via maximum likelihood estimation by minimizing
\begin{align}
\mathcal{L}(r) = - \sum_{(x,y,z,w) \in \calD_L} \log \sigma( w (r(x, y) - r(x, z))).\label{eq:rlhf-loss}
\end{align}
Let $\widehat{r}$ denote the obtained reward model.
The LM policy $\pi$ is then trained to maximize the obtained reward under the KL-regularization:
\begin{align}
\E_{x \sim \rho}[ \E_{y \sim \pi(y\mid x)}[\widehat{r}(x, y)] - \tau D_\mathrm{KL}(\pi \parallel \piref)],\label{eq:rlhf}
\end{align}
where $\rho$ is a distribution over the context; $\piref$ is the reference policy, typically the SFT policy used to initialize RLHF; and $\tau$ is a parameter controlling the deviation from $\piref$.

Direct Preference Optimization (DPO) is an alternative to RLHF that directly optimizes the LM policy from the preference dataset. 
The optimal policy that maximizes \eqref{eq:rlhf} is:
\begin{align}
\pi_{r}(y\mid x) = \frac{1}{Z_{r, \piref}(x)}\piref(y\mid x) \exp( \tau^{-1} r(x, y)), \label{eq:optpi}
\end{align}
where 
\begin{align}
Z_{r, \piref}(x) = \sum_{y} \piref(y\mid x) \exp( \tau^{-1} r(x, y)) . 
\end{align}
Therefore, an LM policy $\pi$ can be viewed as the optimal policy under the reward function:
\begin{equation}
 r(x, y) = \tau\log \frac{\pi(y\mid x)}{\piref(y\mid x)} + \tau\log Z_{r, \piref}(x) .
\end{equation}
By substituting this expression into \eqref{eq:rlhf-loss}, one obtains the corresponding DPO objective. It is known that the optimal policies for RLHF and DPO coincide \citep{ipo}, and several variants of DPO have been proposed in the literature \citep{ipo,spo,kto}. 

\section{Related Works}

Most prior work on data poisoning attacks in machine learning focuses on supervised learning settings, particularly regression or classification tasks \citep{shayegani2023survey}. Typical poisoning objectives include degrading model performance or injecting backdoor triggers. In these settings, attackers can add malicious data to the original dataset, training the victim model on the combined dataset to induce the desired malicious behavior. In contrast, in our setting, the attacker can only flip preference labels in the dataset, making the attack surface more constrained.

Within the LLM alignment pipeline, poisoning attacks have been studied in both the SFT phase \citep{wan2023poisoning,Shu2023neurips} and the RLHF phase \citep{wu2024preference,rlhfpoison,pathmanathan2024poisoning,baumgartner2024best,pathmanathan2024advbdgen,tramer2024universal}. In the SFT phase, human annotators are expected to produce high-quality outputs $y$ for given contexts $x$, creating a natural risk of maliciously crafted pairs $(x, \tilde{y})$. In the RLHF phase, annotators evaluate candidate outputs $y, z$ in a given context $x$ and provide preference labels. While a strong adversary may replace or inject malicious triplets $(x, y, z)$ into the preference dataset \citep{baumgartner2024best,pathmanathan2024advbdgen,tramer2024universal}, arguably the most realistic scenario involves a malicious annotator who can only manipulate the preference label $w$ while the triplet itself remains unchanged \citep{wu2024preference,rlhfpoison,pathmanathan2024poisoning}. This setting is analogous to label-flipping attacks \citep{xiao2012ecai}, but while traditional label-flipping attacks focus on classification, our context involves reward model learning from paired preference data, introducing a distinct problem structure.

The most relevant prior works are \citet{wu2024preference} and \citet{rlhfpoison}, who consider the same setting and develop attack algorithms under both white-box and black-box scenarios. Their empirical investigations reveal that RLHF alignment is vulnerable to label-flipping attacks and that existing defense strategies provide limited protection. However, while these studies empirically demonstrate the vulnerability of RLHF pipelines, our work is the first, to the best of our knowledge, 
to provide a theoretical analysis of the minimal cost of label-flipping attacks required to guide the reward model toward an attacker-specified target. Moreover, we introduce a convex (or linear) programming framework that not only characterizes these costs but also enables practical reduction of attack costs in existing poisoning methods. Specifically, given a candidate malicious dataset, our framework can propose an alternative dataset that induces the same reward function while requiring fewer label flips, highlighting a fundamental and previously unquantified vulnerability in RLHF/DPO pipelines.

\section{Threat Model}

\paragraph{Victim}

The victim has access to a labeled dataset constructed from an unlabeled dataset $\mathcal{D}_U = \{(x_i, y_i, z_i)\}_{i=1}^{N}$, where $x_i \in \mathcal{X}$ is a context, and $y_i, z_i \in \mathcal{Y}$ are two candidate outputs. Each triplet is labeled by external annotators, with the label $w_i = 1$ if $y_i$ is preferred to $z_i$ in context $x_i$, and $w_i = -1$ otherwise. We denote by $\eta(x, y, z)$ the probability that the label for $(x, y, z)$ is $w = 1$. The set of labeled data $(x, y, z, w)$ forms the labeled dataset $\mathcal{D}_L$. Without loss of generality, we assume anti-symmetry of $\eta$, i.e., $\eta(x, y, z) = 1 - \eta(x, z, y)$. In case that multiple ($m \geq 1$) annotations are provided for each triplet, the size of the labeled dataset becomes $m \cdot N$. 

The victim aims to optimize a policy $\pi$ under the labeled dataset $\calD_L$ using RLHF. First, a reward model $r$ is trained by minimizing the empirical loss \eqref{eq:rlhf-loss}.
Once the reward model is trained, the LM policy is optimized to maximize the expected reward under KL regularization as in \eqref{eq:rlhf}. For technical completeness, we assume $\pi_{\mathrm{ref}}(y \mid x) > 0$ for all $(x, y) \in \mathcal{X} \times \mathcal{Y}$. The optimal policy under $r$ is thus given by \eqref{eq:optpi}.

Following \citet{rlhf}, we model the reward function using a pre-trained baseline LLM (typically a fine-tuned model) from which the final unembedding layer is removed to obtain an embedding $\phi: \mathcal{X} \times \mathcal{Y} \to \mathbb{R}^n$. A linear output layer is then added, resulting in the reward model $r(x, y) = \rv^\top \phi(x, y)$, where $\rv \in \mathbb{R}^n$ is referred to as the \emph{reward vector}. 
During the reward training phase, we consider two settings: (1) only the reward vector $\rv$ is trained while the embedding $\phi$ is fixed, and (2) both the reward vector $\rv$ and the embedding $\phi$ are trained.

\paragraph{Attacker}
The attacker's objective is to guide the victim's optimal policy toward a target policy $\pi_A$ with minimal cost. Due to the structure of the optimal policy \eqref{eq:optpi}, only policies that are optimal under reward functions representable by the victim's reward model are valid. Thus, the attack reduces to steering the reward model to a target reward function $r_A$ at minimum cost. The target reward function may be hand-crafted or obtained using attack methods such as \citet{wu2024preference,rlhfpoison}. In the latter case, the objective becomes minimizing the cost required to achieve the same poisoning effect.

We assume the attacker has access to the labeled dataset $\mathcal{D}_L$. The attacker can flip labels $w_i$ to $w_i^A \in \{-1, 1\}$ at a cost defined later. 
However, the attacker cannot modify the input $x$ or candidate outputs $y, z$ in the dataset.

\paragraph{Discrepancy Between Theory and Practice}

For theoretical analysis in the following section, we idealize both the attacker’s capabilities and the victim’s training process. 
We assume the attacker can directly modify the annotation probability $\eta$ rather than flipping individual labels $w$ (i.e., the attacker's action is not binary $\{-1, 1\}$, but continuous $[0, 1]$), and that the victim minimizes the expected loss under $\eta$ (i.e., population version) rather than the empirical loss \eqref{eq:rlhf-loss}. In practice, due to the finite dataset, realizable annotation probabilities are discrete, being multiples of the reciprocal of the count $m$ of each $(x, y, z)$ in the dataset (i.e., $m$ is the number of annotations per datum).
These idealizations enable us to derive theoretical guarantees, while we acknowledge the risk of deviation from practical scenarios, evaluated empirically.

\section{Minimum Cost Attack}

The objective of this study is to identify the cost of the attacker to realize the target policy $\pi_{r_A}$ by label flipping.
Arguably, the most natural choice for the cost of the attacker is the amount of flipped labels, formulated as
\begin{align}
&\text{(practical)} \sum_{(x_i, y_i, z_i, w_i) \in \calD_L} \abs{w_i - w_i^A}  \\
&\text{or (ideal)}\sum_{(x_i, y_i, z_i) \in \calD_U} \abs{\eta(x_i, y_i, z_i) - \eta_A(x_i, y_i, z_i)},\label{eq:l1cost}
\end{align}
where $w_i^A$ is the label after the attack and $\eta_A(x_i, y_i, z_i)$ is the probability of $w_i^A$ being $1$.

More generally, the cost can be measured by using a norm $\norm{\cdot}$. 
From now on, we focus on the ideal situation where the labels $w_i$ follow the probabilities $\eta(x_i, y_i, z_i)$ and the loss function is defined by the expectation of $\mathcal{L}(r)$ with respect to $w_i$, denoted as $\mathcal{L}(r; \eta)$. 
The attack target $r_A$ may be given explicitly, or derived by some other poisoning attack. 
Let $\theta_O$ be a vector of dimension $N$ whose elements are $\eta(x_i, y_i, z_i)$ for $(x_i, y_i, z_i, w_i) \in \calD_L$ and $\theta_A$ a vector of dimension $N$ whose elements are $\eta_A(x_i, y_i, z_i)$. 
The attack cost measured by $\norm{\cdot}$ is defined as
$\norm{\theta_A - \theta_O}$, which recovers \eqref{eq:l1cost} if the norm is the $\ell_1$-norm.

The attacker tries to minimize the cost while realizing the desired policy $\pi_{r_A}$. 
Because of the form of the optimal policy \eqref{eq:optpi}, different reward functions can lead to the same optimal policy. For example, $r(x, y)$ and $r(x, y) + R(x)$ for any $R: \mathcal{X} \to \R$ admit the same optimal policy. Let $\mathcal{R}(\pi) = \{r : \pi_r = \pi\}$ be the set of reward functions for which the optimal policy is $\pi$. 
The attacker's cost minimization problem is formulated with a vector representation $\theta$ of $\eta$ as follows:
\begin{subequations}
\begin{align}
\min_{\theta}\quad & \norm{\theta - \theta_O},\label{eq:prob-gen-a}
\\
\mathrm{s.t.}\quad & \argmin_r \mathcal{L}(r; \theta) \subseteq \mathcal{R}(\pi_{r_A}),\label{eq:prob-gen-b}
\\
&
\norm{2 \theta - \mathbf{1}}_{\infty} \leq 1,\label{eq:prob-gen-c}
\end{align}\label{eq:prob-gen}%
\end{subequations}%
where, with abuse of notation, $\mathcal{L}(r; \theta)$ means $\mathcal{L}(r; \eta)$, and \eqref{eq:prob-gen-c} is introduced to ensure $\eta \in [0, 1]$. 
In the following, we investigate the minimum cost of the poisoning attack theoretically. In particular, we are interested in the lower and upper bounds of the minimum cost to realize the target reward function. 

\subsection{Fixed Embedding}\label{sec:fixed_embedding}

First, we consider the situation where the embedding $\phi$ is fixed during the training. 
The proofs for the theoretical results in this section can be found 
in Appendix. 

The optimization problem \eqref{eq:prob-gen} can be infeasible in cases that the loss function $\mathcal{L}$ admits multiple minimum solutions that lead to different optimal policies. 
However, we can derive that the optimal reward function for the loss \eqref{eq:rlhf-loss} is uniquely determined, showing that \eqref{eq:prob-gen} is valid, under a mild assumption. 
Hereafter, we consider the fixed embedding situation. 
Let $\Phi$ be a matrix of dimension $n \times N$ whose $i$-th column is $\phi(x_i, y_i) - \phi(x_i, z_i)$.
Its Moore-Penrose pseudo-inverse is denoted as $\Phi^\dagger$.
The row and column spaces of $\Phi$ are denoted as $\operatorname{row}(\Phi)$ and $\operatorname{col}(\Phi)$, respectively. 

To proceed theoretical analysis, we assume the following, which is naturally satisfied if $n < N$.\footnote{It holds when $\Phi$ is of full row rank, which occurs with probability 1 for random $\Phi$. We confirmed that all LLM feature matrices used in the experiments satisfied this condition.}
Intuitively, it implies that if two reward functions agree on all data points in $\calD_U$, they must agree everywhere up to a context-dependent offset $R(x)$, meaning that the reward function is fully determined by its values on the dataset.
\begin{assumption}\label{asm}
For any $(x, y, z) \in \mathcal{X}\times\mathcal{Y}\times\mathcal{Y}$, $\phi(x, y) - \phi(x, z) \subseteq \operatorname{col}(\Phi)$. 
\end{assumption}

The optimization problem \eqref{eq:prob-gen} can be transformed as a convex optimization problem.
\begin{theorem}\label{thm:cvxopt}
Suppose that \Cref{asm} holds.
Let $\zeta = \theta - \theta_O$. 
Then, the minimum cost poisoning attack problem \eqref{eq:prob-gen} is equivalently formulated as a convex optimization problem with linear equality and inequality conditions:%
\begin{subequations}%
\begin{align}%
\min_{\zeta} \ \norm{\zeta} 
\quad \mathrm{s.t.} \quad & \Phi \zeta = \Phi (\theta_A - \theta_O), \\
& - \theta_O \leq \zeta \leq (\bm{1} - \theta_O), 
\end{align}%
\label{eq:primal-linear}%
\end{subequations}%
where $\leq$ denotes element-wise comparison, as used hereafter.
\end{theorem}%

That is, by solving the convex optimization problem \eqref{eq:primal-linear} and letting $\zeta^*$ be its optimal solution, one can obtain the preference probability $\theta_A^* = \theta_O + \zeta^*$ that leads to the same target reward function in $\mathcal{R}(\pi_{r_A})$ as $\theta_A$ with a reduced or equal cost $\norm{\theta_A^* - \theta_O} \leq \norm{\theta_A - \theta_O}$. It is irrelevant to the way of crafting $\theta_A$.

This optimization problem has a convex objective function and linear constraints. Therefore, one can obtain a solution to this problem, $\zeta^*$, by using a standard convex optimization solver. In case of the $\ell_1$ attack cost, this can be reformulated as a linear programming problem by decomposing $\zeta$ as $\zeta = \zeta_+ - \zeta_-$ where $\zeta_+, \zeta_- \in \R_{+}^{N}$ and rewriting $\norm{\zeta}$ as $\norm{\zeta} = \bm{1}_{N}^\T (\zeta_+ + \zeta_-)$. Therefore, one can employ a linear programming solver to obtain the optimum solution. 
See Appendix for details. 

By considering the Lagrangian dual problem of the primal problem \eqref{eq:primal-linear}, we can derive the minimum attack cost bounds.
It is derived as follows. 
The primal problem is a convex optimization problem with linear equality and inequality constraints. 
Therefore, if a relaxed Slater condition is satisfied, i.e., a feasible solution exists, then the strong duality holds and the solution to the dual problem provides the minimum value of the primal problem. 
In our case, $\zeta = \theta_A - \theta_O$ is a feasible solution.
Hence, the strong duality holds and the maximum value of the dual problem is the minimum value of the primal problem.

Based on the above argument, a lower bound and an upper bound of the minimum cost are derived.
\begin{theorem}\label{lem:lower}
The minimum cost of \eqref{eq:primal-linear} is lower bounded by
\begin{equation}
\frac{ \norm{(\Phi^\dagger \Phi) (\theta_A - \theta_O)}_2^2 }{ \norm{(\Phi^\dagger \Phi) (\theta_A - \theta_O)}_* }.\label{eq:lower}
\end{equation}
\end{theorem}

\begin{theorem}\label{lem:upper}
Let $\theta^* = \theta_O + (\Phi^\dagger \Phi) (\theta_A - \theta_O)$.
Let $\alpha^* = \max\{ \norm{\theta^* - 0.5\cdot \bm{1}}_\infty - 0.5, 0 \}$ and $\bar{\alpha} = 0.5 - \norm{\theta_A - 0.5}$.
The minimum cost of \eqref{eq:primal-linear} is upper bounded by
\begin{equation}
\min\left\{ \norm*{ \left( \frac{\alpha^* I + \bar{\alpha} \Phi^\dagger \Phi}{\alpha^* +\bar{\alpha}}\right) (\theta_A - \theta_O)}, \ \norm{\theta_A - \theta_O} \right\} . \label{eq:upper}
\end{equation}
\end{theorem}

\begin{remark}
The matrix $\Phi^\dagger \Phi$ defines the orthogonal projection from $\R^{N}$ to a subspace spanned by the rows of $\Phi$, whose rank is at most $n$. If the cost is defined by the $\ell_2$ norm, we always have $\norm{ (\Phi^\dagger \Phi) (\theta_A - \theta_O)}_2 \leq \norm{\theta_A - \theta_O}_2$. 
The discrepancy between $\frac{ \norm{(\Phi^\dagger \Phi) (\theta_A - \theta_O)}_2^2 }{ \norm{(\Phi^\dagger \Phi) (\theta_A - \theta_O)}_* }$ in \eqref{eq:lower} and $\norm{ (\Phi^\dagger \Phi) (\theta_A - \theta_O)}$ in \eqref{eq:upper} comes from the primal-dual norm relation, i.e., $\norm{\zeta}_2^2 \leq \norm{\zeta} \norm{\zeta}_*$. 
In this sense, these bounds are tight because the equality holds for some $\zeta$. 
\end{remark}

These theorems indicate that the cost of a poisoning attack can be reduced significantly.
From the defense perspective, it suggests that the victim must be prepared for the attack to guide the reward model arbitrarily in
\begin{equation}
\Theta_{k}^{A} = \left\{ \theta \mid \frac{ \norm{(\Phi^\dagger \Phi) (\theta - \theta_O)}_2^2 }{ \norm{(\Phi^\dagger \Phi) (\theta - \theta_O)}_* } \leq k \right\}
\end{equation}
if $k$ data points are annotated by an untrusted annotator.
It can be significantly wider than the set corresponding to the naive attack with cost no greater than $k$, i.e.,
\begin{equation}
\widetilde{\Theta}_{k}^A = \left\{ \theta \mid  \norm{\theta - \theta_O} \leq k \right\},
\end{equation}
in particular, when the rank of $\Phi^\dagger \Phi$ is significantly smaller than its dimension, which corresponds to the situation that the number $n$ of features is significantly smaller than the number $N$ of data points. It helps us assess the security risk of allowing untrusted individuals to annotate data.

We investigate the influence of the choice of the embedding $\phi$ (i.e., $\Phi$) on the minimum attack cost.
\Cref{lem:phi} states that, if the feature extractor is fixed, then the more capable the feature extractor's representation power is, the more cost the attacker needs to spend to realize the same target reward.
It suggests that a greater number $n$ of features results in models that are more robust against label flipping attacks.
\begin{proposition}\label{lem:phi}
Suppose that $\phi_1 : \mathcal{X} \times \mathcal{Y} \to \R^{n_1}$ and $\phi_2 : \mathcal{X} \times \mathcal{Y} \to \R^{n_2}$ satisfies \Cref{asm} and $\operatorname{row}(\Phi_1) \subseteq \operatorname{row}(\Phi_2)$.
If the target reward function is $r_A = \rv_{A,1}^\T \phi_1 = \rv_{A,2}^\T \phi_2$, then a feasible solution $\zeta_2$ to the problem under $\phi_2$ is also feasible under $\phi_1$.
Therefore, the minimum value of the problem \eqref{eq:primal-linear} under $\Phi_1$ is no greater than that under $\Phi_2$.
\end{proposition}

\subsection{Adaptive Embedding}\label{sec:adaptive_embedding}

Now we consider a more general reward model 
\begin{equation}
r(x, y) = \inner{\rv}{\phi_{\omega}(x, y)},\label{eq:nn-reward}
\end{equation}
where $\phi$ is an adaptive embedding parameterized by $\omega$. 
Suppose that the attacker's target reward is expressed as
\begin{equation}
r_A(x, y) = \inner{\rv_A}{\phi_{\omega_A}(x, y)}\label{eq:ra_nn}
\end{equation}
with the reward vector $\rv_A$ and the parameter $\omega_A$ of the embedding.
In this situation, however, the optimization problem \eqref{eq:prob-gen} is not always well-defined as the reward function minimizing $\mathcal{L}(r; \eta)$ may not be uniquely determined. 
In such situations, whether the attack succeeds or not depends on which solution the victim's reward model converges to. Therefore, it may depends on the initial value of the victim's reward model and its learning algorithm. 
To make the optimization problem feasible, we relax the notion of attack success as having a potential to obtain the target reward function.
The relaxed optimization problem is formulated as follows:
\begin{subequations}%
\begin{align}%
\min_{\eta}\quad & \norm{\theta - \theta_O},
\\
\mathrm{s.t.}\quad & r_A \in \argmin \mathcal{L}(r; \eta) ,\label{eq:ra}
\\
&
\bm{0} \leq \eta(x, y, z) \leq \bm{1},  \forall x, y, z.
\end{align}\label{eq:prob-relax}%
\end{subequations}%
Considering a lower bound of the minimum cost of such a relaxed problem provides the guarantee from the victim's perspective that the attack cannot be successful without paying a derived cost. 

Similarly to \Cref{thm:cvxopt}, we can transform the attacker's optimization problem \eqref{eq:prob-relax} as follows. 
First, we realize that \eqref{eq:ra} holds if and only if a reward model $r = \rv^\T \phi_{\bar{\omega}}$ satisfying $\rv^\T \Phi_{\bar{\omega}} = \rv_A^\T \Phi_{\omega_A}$ is included in $\argmin \mathcal{L}(r; \eta)$. 
For these reward functions, we have $\theta_{r_A} = \theta_{r}$. 
Choose one such $\bar{\omega}$. 
Then, analogously to the derivation of \Cref{thm:cvxopt}, the attacker's optimization problem reduces to 
\begin{subequations}
\begin{align}
\min_{\zeta} \  \norm{\zeta} \quad
\mathrm{s.t.} \ \ \ & \Phi_{\bar{\omega}} \zeta = \Phi_{\bar{\omega}} (\theta_A - \theta_O), \label{eq:primal-nn:cond}\\
& \bar{\omega} \in \{\omega : \exists \bar{\rv} \ \mathrm{ s.t. }\ \bar{\rv}^\T \phi_{\bar{\omega}} = \rv_A^\T \phi_{\omega_A} \}, \label{eq:primal-nn:omega}\\
& - \theta_O \leq \zeta \leq (\bm{1} - \theta_O).
\end{align}%
\label{eq:primal-nn}%
\end{subequations}
The point here is that the attacker does not need to set $\bar{\omega} = \omega_A$ and may choose $\bar{\omega}$ such that the attack cost is the smallest.
The following result is a straightforward consequence.
\begin{proposition}\label{prop:primal-nn-upper}
The minimum cost of \eqref{eq:primal-nn} is upper bounded by the minimum cost of \eqref{eq:primal-linear} where $\Phi$ is replaced with $\Phi_{\omega_A}$. 
\end{proposition}

\Cref{prop:primal-nn-upper} indicates that we can reduce the cost of the attack to realize the target reward $r_A$ by solving convex (or linear) programming \eqref{eq:primal-linear} with $\Phi = \Phi_{\omega_A}$.
However, we emphasize that, differently from the case of the fixed embedding, it does not provide the minimum cost attack due to the degrees of freedom in \eqref{eq:primal-nn:omega}, and it provides the solution to a ``relaxed'' optimization problem \eqref{eq:prob-relax}. 
Therefore, guarantees from the perspective of attackers are hard to obtain.

Now we consider the worst situation for the victim.
In light of \Cref{lem:phi}, the minimum attack cost is no greater for $\bar{\omega}$ than for $\omega_A$ if $\operatorname{row}(\Phi_{\bar{\omega}}) \subseteq \operatorname{row}(\Phi_{\omega_A})$.
Suppose that the representational capacity of $\phi_\omega$ is high enough that there exists $\bar{\omega}$ such that $\operatorname{col}(\Phi_{\bar{\omega}}) = \{\rv_A^\T \Phi_{\omega_A}\}$.
By assuming the existence of such $\bar{\omega}$, the attacker's optimization problem reads%
\begin{subequations}%
\begin{align}%
\min_{\zeta} \  \norm{\zeta} \quad 
\mathrm{s.t.} \quad & \rv_A^\T \Phi_{\omega_A} \zeta = \rv_A^\T \Phi_{\omega_A} (\theta_A - \theta_O), \label{eq:primal-nn-reduced:cond}\\
& - \theta_O \leq \zeta \leq (\bm{1} - \theta_O), 
\end{align}\label{eq:primal-nn-reduced}%
\end{subequations}%
where we used the fact that $\Phi_{\bar{\omega}} \zeta = \Phi_{\bar{\omega}} (\theta_A - \theta_O)$ is equivalent to $\rv_A^\T \Phi_{\omega_A} \zeta = \rv_A^\T \Phi_{\omega_A} (\theta_A - \theta_O)$. 
This problem can be solved with a convex programming or a linear programming solver as for \eqref{eq:primal-linear} but possibly with reduced cost.

\begin{theorem}\label{thm:nnbound}
Suppose that the attacker's target reward function is expressed as \eqref{eq:ra_nn}. 
Then, the minimum cost of \eqref{eq:primal-nn} is lower bounded by that of \eqref{eq:primal-nn-reduced}.
Moreover, if there exists $\bar{\omega}$ such that $\operatorname{col}(\Phi_{\bar{\omega}}) = \{\rv_A^\T \Phi_{\omega_A}\}$, the minimum cost of \eqref{eq:primal-nn} is equal to that of \eqref{eq:primal-nn-reduced}.
\end{theorem}

The upper and lower bounds for the cost of \eqref{eq:primal-nn-reduced} are derived in \Cref{lem:lower} and \Cref{lem:upper}, respectively. 
The attacker need not $\bar{\omega}$ explicitly; solving \eqref{eq:primal-nn-reduced} requires only the target reward function $r_A$.
However, in reality, considering the minimum cost of \eqref{eq:primal-nn-reduced} as the minimum cost for the attack $\theta_A$ may be too conservative from the defense perspective. Such an attack will not realize the target reward function in practice due to the solution multiplicity, suboptimal optimization, and the fact that a $\bar{\omega}$ satisfying $\operatorname{col}(\Phi_{\bar{\omega}}) = \{\rv_A^\T \Phi_{\omega_A}\}$ does not necessarily exist in general.

\subsection{Practical Post-Processing Method}
Based on the above theoretical analysis, we propose a practical post-processing method to minimize the cost of any existing label-flipping attack. Given a target preference probability vector $\theta_A$, either hand-crafted or generated by an existing attack, we solve the convex optimization problem in \eqref{eq:primal-linear} to obtain a cost-minimized vector $\theta_A^*$ that induces the same target reward model while requiring fewer label flips.
In the case of adaptive embeddings, we use the initial embedding $\phi$ of the reward model to form the optimization problem \eqref{eq:primal-linear} as if it were fixed. A discrepancy exists between our theoretical analysis and practical implementation, as the analysis assumes knowledge of the embedding parameter $\omega_A$ corresponding to the target reward $r_A$. However, as demonstrated later, the practical approach remains effective in reducing the cost without deteriorating the attack's performance.

After obtaining $\theta_A^*$, we discretize it by rounding so that the preference vector takes values in
\begin{equation}
\Theta_m = \left\{ \theta \in \mathbb{R}^N : [\theta]_k = \frac{i}{m} \text{ for } i \in \{0, \dots, m\} \right\},
\end{equation}
where $[\theta]_k$ indicates the $k$th element of $\theta$ and $m$ is the number of annotations per datum, referred to as the granularity. Finally, we flip the preference labels in the dataset to follow the discretized vector. We refer to this post-processing method as \emph{Poisoning Cost Minimization (PCM)}.

Importantly, PCM is agnostic to how the target $\theta_A$ is generated and can be layered onto any label-flipping attack, providing a systematic way to reduce poisoning costs while preserving attack efficacy. We apply this post-processing in our empirical evaluations to demonstrate its effectiveness across synthetic and real LLM alignment datasets.

\section{Numerical Analysis on Synthetic Data}

We demonstrate the tightness of the bounds and how small the minimum cost can be compared to the cost of the naive attack. 
In particular, we show that the minimum cost can be significantly reduced from the naive cost when the number of data points is significantly greater than the number of features. 
For this purpose, we generate synthetic data and compare the cost of the cost minimized attack and the naive attack.

\paragraph{Dataset}
We produce a synthetic dataset $\mathcal{D}_U$ with embeddings $\phi(x_i, y_i)$ and $\phi(x_i, z_i) \in \R^n$ for $(x_i, y_i, z_i) \in \mathcal{D}_U$ generated randomly from the standard normal distribution. 
Without loss of generality, the first response $y_i$ is considered to be the preferred response in the original annotation, i.e., $\theta_O = \mathbf{1}$. 
The dataset size is $N = \abs{\mathcal{D}_U}$.
To simulate the situation where multiple annotators are assigned to provide their preferences for the same tuples, we duplicate each datum $m$ times.
That is, $\theta_O$ as well as annotation probability after the poisoning attack, $\theta_A$, can take values in $\Theta_m$.

\paragraph{Attack Scenario}
By nature of this synthetic dataset, we consider the situation where the embedding is fixed.
We suppose that we have a target attack preference probability $\theta_A \in \Theta_m$. 
The attacker tries to minimize the attack cost measured by $\ell_1$-norm, $\norm{ \theta - \theta_O }_1$. 
We consider two attack targets: 1) $\theta_A$ is generated by flipping each element of $\theta_O$ with probability $0.1$; 2) $\theta_A$ is generated by RLHFPoison \cite{rlhfpoison} with quality filter parameter $a = 0.25$ and final poisoning ratio $b = 0.1$. RLHFPoison originally generates a dataset to make the LLM output a longer response without significantly changing the other aspects. Here, instead of the output length, the first feature of the output is to be maximized. That is, the reward signals for data with greater first feature values are to be maximized. 
For each $\theta_A$, we apply the proposed post-processing, PCM, to obtain the cost-minimized preference probability $\theta_A^*$.

\paragraph{Performance Metric}
We focus on two metrics. The first one is the $\ell_1$-cost $\norm{\theta - \theta_O}_1$ after the discretization. 
Due to the discretization, the performance of the attack by PCM may degrade. 
To measure the performance degradation by PCM, we minimize the loss function \eqref{eq:rlhf-loss} for $r = \rv^\T \phi$ with respect to $\rv$, where the preference labels are given by $\theta_O$, $\theta_A$, and $\theta_A^*$ (discretized). Letting the optimal reward functions obtained with the above preference datasets be denoted as $r_O$, $r_A$, and $r_A^*$. Then, we compute the performance loss rate
\begin{equation}
\frac{\sum_{i=1}^{N} \abs{\sigma(r_A^*(x, y) - r_A^*(x, z)) - \sigma(r_A(x, y) - r_A(x, z))} }{\sum_{i=1}^{N} \abs{\sigma(r_A(x, y) - r_A(x, z)) - \sigma(r_O(x, y) - r_O(x, z))}}.\label{eq:plr}
\end{equation}
It measures the average preference difference between the trained reward models using the target $\theta_A$ and its cost-minimized $\theta_A^*$ relative to that between the trained reward models using $\theta_A$ and the original $\theta_O$.

\begin{figure}[t]
\begin{subfigure}{\hsize}%
\includegraphics[width=\hsize,clip,trim=30 0 10 0]{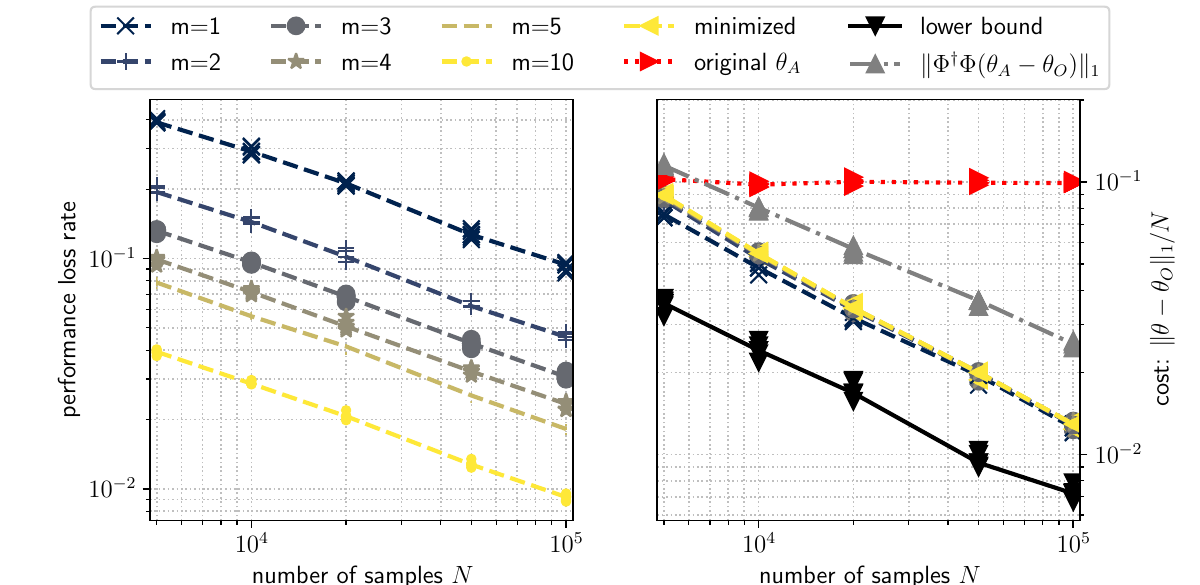}%
\caption{\#features $n = 1000$}
\end{subfigure}%
\\%
\begin{subfigure}{\hsize}%
\includegraphics[width=\hsize,clip,trim=30 0 10 45]{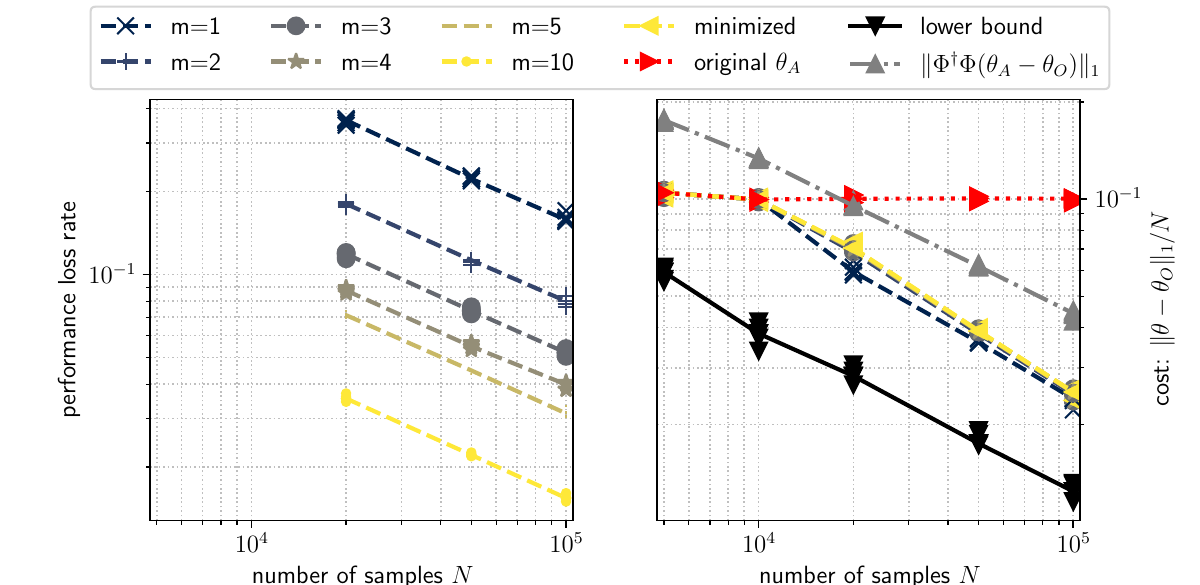}%
\caption{\#features $n = 3000$}
\end{subfigure}%
\caption{Cost (right) and performance loss rate \eqref{eq:plr} (left) of the proposed cost minimization, PCM, for random flip attack. 
Results of 5 trials (points) as well as their median (lines).
Minimized: the cost of $\theta_A^*$ before discretization, Original: the cost of $\theta_A$, Lower bound: \eqref{eq:lower}, $\norm{\Phi^\dagger\Phi (\theta_A - \theta_O)}_1$: a term appearing in the upper bound \eqref{eq:upper}. The other lines are the performance loss rate and the cost of the proposed attack with discretization using different granularity $m$. 
Missing data points in the preference loss rate indicate no performance loss because $\theta_A = \theta_A^*$ (no cost reduction as well).
}\label{fig:random}
\end{figure}

\begin{figure}[t]
\begin{subfigure}{\hsize}%
\includegraphics[width=\hsize,clip,trim=30 0 10 0]{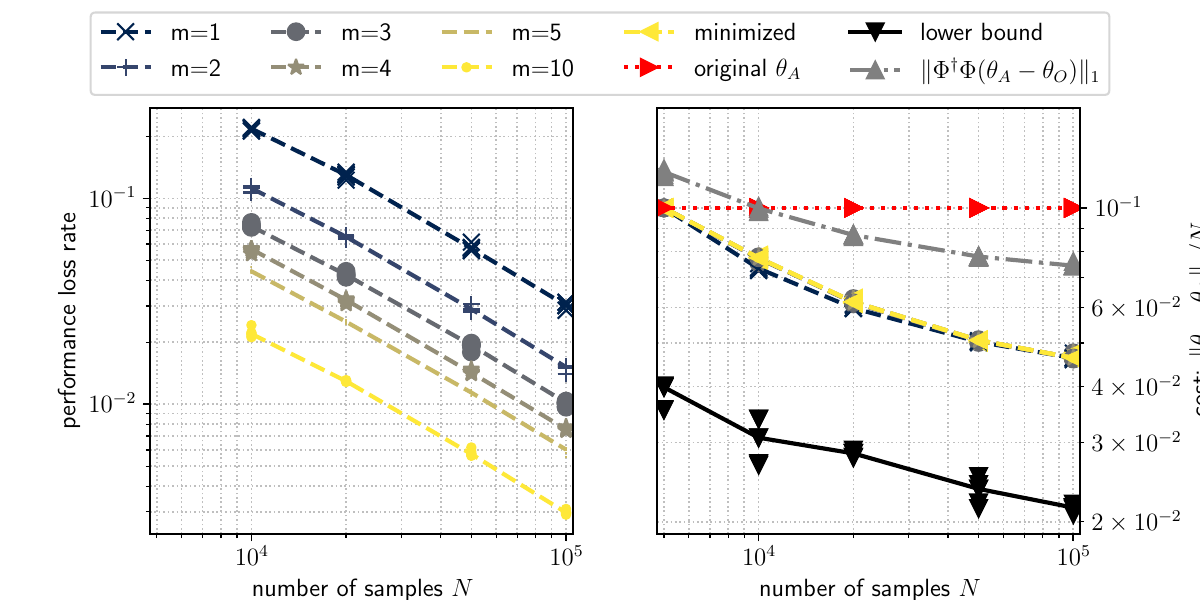}%
\caption{\#features $n = 1000$}
\end{subfigure}%
\\%
\begin{subfigure}{\hsize}%
\includegraphics[width=\hsize,clip,trim=30 0 10 50]{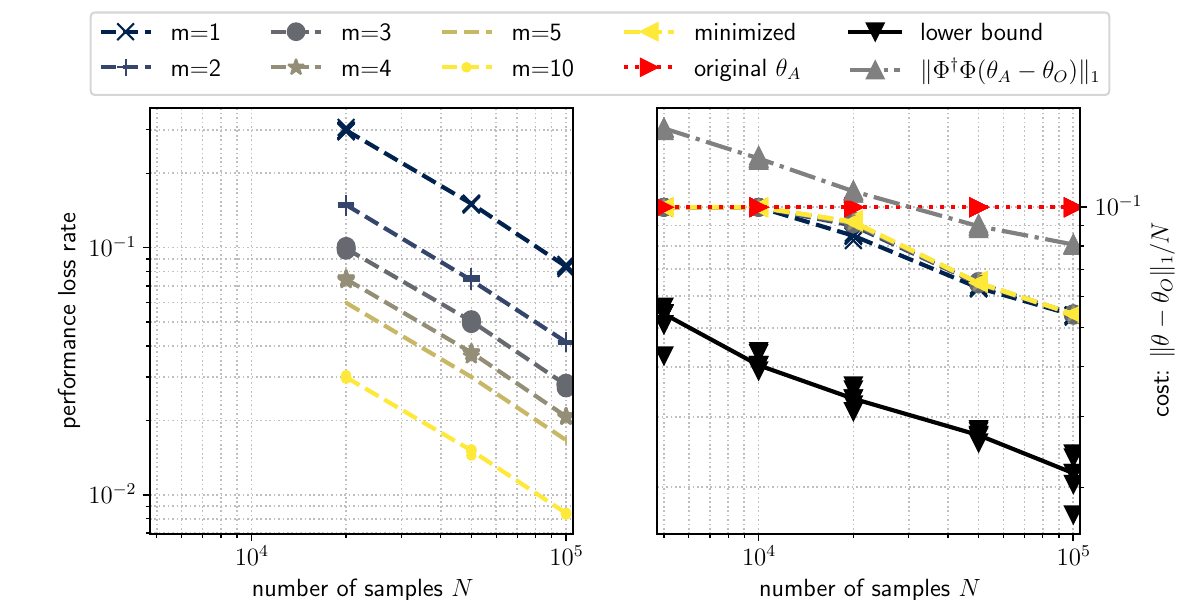}%
\caption{\#features $n = 3000$}
\end{subfigure}%
\caption{Cost (right) and performance loss rate (left) of the proposed cost minimization, PCM, for RLHFPoison attack. See the caption of \Cref{fig:random} for details.}\label{fig:rlhfpoison}
\end{figure}

\paragraph{Results}
The results are shown in \Cref{fig:random,fig:rlhfpoison}.
Although the algorithms are all deterministic, the datasets are randomly generated. Therefore, we performed 5 trials with different dataset generations.
The findings are summarized as follows.
(1) Though it is not guaranteed by \Cref{lem:upper}, $\norm{(\Phi^\dagger \Phi)(\theta_A - \theta_O)}_1$ provides a good upper bound of the proposed scheme when it is smaller than the naive cost $\norm{\theta_A - \theta_O}_1$. The results all fit between the lower bound provided in \Cref{lem:lower} and this value, and the discrepancy between them is around the factor of $3$ to $4$.
(2) There is a higher chance to reduce the attack cost for a larger data set, i.e., greater $N$ if the original attack cost $\norm{\theta_A - \theta_O}_1$ is more or less constant (i.e., the flip rate is fixed). It may be understood as that the rank of $\Phi^\dagger \Phi$ ($\leq n$) will be smaller than its dimension ($N$).
(3) The cost does not depend heavily on the granularity $m$, but a smaller performance loss rate can be achieved with a greater $m$. It is intuitive because the cost minimization process does not affect the trained reward function if no discretization is performed. Even with $m = 1$ (i.e., each data point is annotated by a single annotator), the performance loss rate can be around $0.1$ if $N$ is sufficiently large.
(4) The cost reduction effect increases linearly in the dataset size $N$ for the random-flip attack, whereas its scaling is lower for the attack by RLHFPoison. Nevertheless, a similar trend (the cost can be reduced when $N \gtrsim 5n$) is observed.

\section{Evaluation on Public LLMs and Open Dataset}

We demonstrate the cost reduction achieved by the proposed framework across different LLMs on publicly available datasets. 
In particular, we show that the proposed method remains effective even when the entire LLM is trained using DPO (adaptive embedding scenario), despite the framework itself being derived under the fixed embedding assumption.

\paragraph{Datasets and Models}
We employ three datasets of varying size and three models of varying size.
The datasets are \textsc{social-reasoning-rlhf} ($N = 3,820$) \citep{prolificai}, \textsc{pku-saferlhf} ($N = 73,907$) \cite{ji2024beavertails}, and \textsc{hh-rlhf} ($N=160,800$) \citep{bai2022traininghelpfulharmlessassistant}. Since \textsc{social-reasoning-rlhf} dataset does not contain a test set, we used 20\% of the training data for testing.  
The models are \texttt{Phi-3.5-mini-instruct} ($n=3072$) \citep{phi3}, \texttt{LLaMA-2-7b} ($n=4096$), and \texttt{LLaMA-2-13b} ($n=5120$) \citep{llama2}. 



\paragraph{Attack Scenario}

The attack target is generated using RLHFPoison, which aims to increase the output length of an LLM without significantly affecting other behavioral characteristics. 
The quality filter parameter and final poisoning ratio are set to $a = 0.25$ and $b = 0.05$, respectively. Consequently, $5\%$ of the preference labels in the training dataset are flipped.
PCM is applied after obtaining $\theta_A$ via RLHFPoison. 

\paragraph{Performance Metric}

LLMs are trained by DPO (3 epochs with $\tau = 0.1$ and learning rate $10^{-6}$) with the original preference $\theta_O$, the malicious preference $\theta_A$ generated by RLHFPoison, and the preference $\theta_A^*$ generated by PCM from $\theta_A$. 
We measure the cost (flip rate) reduction rate of $\theta_A^*$ over $\theta_A$ ($5\%$), i.e., flip rate of $\theta_A^*$ divided by $0.05$. 
Moreover, we measure the average output length of LLMs on the test dataset to assess the performance drop due to the cost minimization.
As the output length significantly varies over contexts, we standardize each output length $\ell$ with the corresponding output length $\ell_O$ of the LLM trained on the original preference by $(\ell - \ell_O) / \ell_O$. We call it the output length increase rate.

\paragraph{Results}

We confirm that PCM is still effective in this practical scenario,  summarized in \Cref{tbl:dporesult}. 
Although the RLHFPoison+PCM resulted in reducing the output length increase rates compared to RLHFPoison itself on HH-RLHF dataset, it keeps the effect of increasing the output length. PCM successfully reduces the label flip rate compared to RLHFPoison. As expected, a greater dataset size allows more cost reduction. Meanwhile, no cost reduction effect can be observed on \textsc{social-reasoning-rlhf}, where the number $n$ of features of the models is greater than the number of training data, hence the results are omitted. 
Further details and results are provided 
in Appendix. 
\begin{table}[t]
\centering
\begin{tabular}{lll}
  \toprule
               & RLHFPoison & RLHFPoison+PCM \\
  \midrule
    \multicolumn{3}{c}{PKU-SafeRLHF}\\
  \midrule
Phi-3.5-mini   & $0.44 \pm 0.01$     & $0.40 \pm 0.01$  ($-13.4\%$) \\
Llama-2-7b   & $0.29 \pm 0.02$   & $0.29 \pm 0.01$     ($-10.6\%$) \\ 
Llama-2-13b  & $0.25 \pm 0.01$     & $0.37 \pm 0.01$     ($-8.2\%$)
   \\
  \midrule
    \multicolumn{3}{c}{HH-RLHF}\\
  \midrule
    Phi-3.5-mini   & $0.55 \pm 0.02$    &  $0.27 \pm 0.02$  ($-30.4\%$) \\
Llama-2-7b   & $1.08 \pm 0.36$    & $0.87 \pm 0.05$      ($-29.8\%$) \\ 
Llama-2-13b  & $1.63 \pm 0.55$    & $1.27 \pm 0.15$   ($-20.0\%$)\\
   \bottomrule
\end{tabular}
\caption{Average $\pm$ standard error of output length increase rate for RLHFPoison and RLHFPoison+PCM, and cost reduction rate for PCM (in parenthesis).}\label{tbl:dporesult}
\end{table}


\section{Discussion}

This work establishes a theoretical foundation for understanding the vulnerability of RLHF/DPO pipelines to label-flipping poisoning attacks. Our theoretical results include: lower and upper bounds for the minimum attack cost in the fixed embedding case (\Cref{lem:lower}, \Cref{lem:upper}); the implication that attacks on models with smaller feature dimensions may succeed at a smaller cost than those targeting models with greater feature dimensions (\Cref{lem:phi}, fixed embedding case); the finding that attacks in the adaptive embedding scenario are no more difficult than in the fixed embedding scenario if the target embedding is known (\Cref{prop:primal-nn-upper}); and that attacks can succeed with a significantly small cost in the worst-case adaptive embedding scenario (\Cref{thm:nnbound}). As a byproduct of this analysis, we propose a general post-processing method to minimize attack cost. This framework can be combined with any label-flipping attack to reduce its cost while preserving the intended poisoning effect, thereby significantly improving the efficiency of existing attacks and contributing to more effective stress-testing for red-teaming efforts.

We conclude by outlining the limitations of the current study and identifying promising future directions. Our current analysis relies on idealizations---assuming optimal reward model recovery, exact attacker knowledge of reward function structure, and direct preference probability modification---while practical factors are not theoretically accounted for, despite empirical validation of our cost minimization on synthetic data. Future work should tackle these discrepancies, possibly by evaluating performance loss (deviation from the target reward function) to deepen our understanding of vulnerability. Furthermore, while our adaptive embedding analysis reveals crucial worst-case scenarios, the results (e.g., \Cref{thm:nnbound}) are conservative; evaluating performance loss under realistic assumptions such as bounded embedding changes represents another promising avenue. These identified areas highlight that a significant contribution of this work lies in opening up numerous critical directions for future research.

\section{Acknowledgements}
This work was partially supported by JSPS KAKENHI Grant Number 23H00483 and JST K Program Japan Grant Number JPMJKP24C3.

\bibliography{refs}

\clearpage
\onecolumn
\appendix

\section{Proofs}
\subsection[Proof of Theorem~\ref{thm:cvxopt}]{Proof of \Cref{thm:cvxopt}}\label{apdx:thm:cvxopt}

To prove \Cref{thm:cvxopt}, we show the following lemmas, whose proofs are provided in the following subsections.
\begin{lemma}\label{lem:welldefined}
If \Cref{asm} holds, then $\rv_{1}^\T \Phi = \rv_{2}^\T \Phi$ implies that $\rv_1^\T(\phi(x, y) - \phi(x, z)) = \rv_2^\T(\phi(x, y) - \phi(x, z))$ for all $(x, y, z) \in \mathcal{X} \times \mathcal{Y} \times \mathcal{Y}$. 
\end{lemma}
\begin{lemma}\label{lem:BTequal}
For two reward functions $\bar{r}$ and $r$, $r \in \mathcal{R}(\pi_{\bar{r}})$ if and only if we have for any $(x, y, z)$, 
\begin{equation}
r(x, y) - r(x, z) = \bar{r}(x, y) - \bar{r}(x, z).\label{eq:BTequal}
\end{equation}
Equivalently, letting $\eta_{r}$ be the preference probability associated with $r$, defined as $\eta_r(x, y, z) = \sigma(r(x, y) - r(x, z))$, we have 
$\mathcal{R}(\pi_{\bar{r}}) = \{r : \eta_r = \eta_{\bar{r}}\}$. 
\end{lemma}
%
\begin{lemma}\label{lem:rset}
If \Cref{asm} holds, then for a given $r(x, y) = \rv^\T\phi(x, y)$, 
\begin{equation}
\mathcal{R}(\pi_{r}) = \{r' : \eta_{r'}(x, y, z) = \eta_{r}(x, y, z), \forall (x, y, z) \in \mathcal{D}_U\}
.\label{eq:Rset}
\end{equation}
\end{lemma}
\begin{lemma}\label{lem:reward-linear}
Let $\bar{\rv}$ be a reward vector and $\theta_A$ a vector of dimension $N$ whose $i$-th element is $\bar{\eta}(x_i,y_i,z_i) = \sigma(\inner{\bar{\rv}}{\phi(x_i, y_i) - \phi(x_i, z_i)})$.
Then, $\bar{\rv}$ is a solution to $\argmin_{\rv} \mathcal{L}(\rv^\T\phi; \theta)$ if and only if $\theta$ satisfies
\begin{equation}
\Phi \theta = \Phi \theta_A,\label{eq:r-linear-cond}
\end{equation}
Moreover, if \Cref{asm} holds and $\eta$ satisfies \eqref{eq:r-linear-cond}, then 
\begin{equation}
\argmin_{r} \mathcal{L}(r; \theta) = \{\bar{r} = \bar{\rv}^\T\phi \}.
\end{equation} 
\end{lemma}

Utilizing these lemmas, we prove \Cref{thm:cvxopt} as follows.
\begin{proof}
Let $\theta_r$ denote the vector representation of the preference probability $\eta_r$ defined by the reward function $r$.
\Cref{lem:rset} states that $\mathcal{R}(\pi_{r_A}) = \{r: \theta_{r} = \theta_{r_A}\}$. 
Because $\theta_A$ is the preference probability that leads to the target reward function $r_A$, we have $\theta_A = \theta_{r_A}$
That is, for any reward function included in $\mathcal{R}(\pi_{r_A})$, the vector representation of the corresponding probability is $\theta_{A}$. \Cref{lem:reward-linear} states that $\argmin_{r} \mathcal{L}(\rv^\T\phi; \theta) = \{r: \Phi \theta_r = \Phi \theta\}$.
Therefore, \eqref{eq:prob-gen-b} is satisfied if and only if $\theta$ is such that $\{r: \Phi \theta_r = \Phi \theta\} \subseteq \{r: \theta_r = \theta_{A}\}$. It reduces to select $\theta$ such that $\Phi \theta_A = \Phi \theta$. 
It is equivalent to $\Phi (\theta_A - \theta_O) = \Phi (\theta - \theta_O)$.
The rest is trivial. This completes the proof.
\end{proof}

\subsection{Proof of \Cref{lem:welldefined}}\label{apdx:lem:welldefined}
\begin{proof}
Suppose that $\rv_{1}^\T \Phi = \rv_{2}^\T \Phi$.
It implies that $\rv_{2} = \rv_{1} + (I -  \Phi \Phi^\dagger) \Delta$ for some $\Delta \in \R^{n}$. 
Because of \Cref{asm}, $\phi(x, y) - \phi(x, z) \in \operatorname{col}(\Phi)$ for any $(x, y, z) \in \mathcal{X} \times \mathcal{Y}\times \mathcal{Y}$, implying that $(I -  \Phi \Phi^\dagger) (\phi(x, y) - \phi(x, z)) = 0$. 
Therefore, 
\begin{subequations}
\begin{align}
\rv_2^\T(\phi(x, y) - \phi(x, z))
&= (\rv_1 + (I -  \Phi \Phi^\dagger) \Delta)^\T (\phi(x, y)  - \phi(x, z)) \\
&= \rv_1^\T (\phi(x, y)  - \phi(x, z)) + \Delta^\T (I -  \Phi \Phi^\dagger) (\phi(x, y)  - \phi(x, z)) \\
&= \rv_1^\T (\phi(x, y)  - \phi(x, z))
\end{align}%
\end{subequations}%
for any $(x, y) \in \mathcal{X} \times \mathcal{Y}$.
\end{proof}

\subsection{Proof of \Cref{lem:BTequal}}\label{apdx:lem:BTequal}

\begin{proof}
The optimal policies under the reward model $r$ and $\bar{r}$ are coincide if and only if for all $(x, y, z)$,
\begin{align}
\frac{\pi_r(y \mid x)}{\pi_{r}(z \mid x)}
= \frac{\pi_{\bar{r}}(y \mid x)}{\pi_{\bar{r}}(z \mid x)}.
\end{align}
Because of the form of the optimal policy \eqref{eq:optpi}, 
it is equivalent to have
\begin{align}
r(x, y) - r(x, z) = \bar{r}(x, y) - \bar{r}(x, z).
\end{align}
This completes the proof.
\end{proof}

\subsection{Proof of \Cref{lem:rset}}\label{apdx:lem:rset}

\begin{proof}
Suppose that $\pi_r = \pi_{r'}$. 
Then, in light of \Cref{lem:BTequal}, 
\begin{subequations}
\begin{align}
\pi_r = \pi_{r'}
\iff &(r(x, y) - r(x, z)) - (r'(x, y) - r'(x, z)) = 0 \quad \forall (x, y, z) \in \mathcal{X}\times\mathcal{Y}\times \mathcal{Y}
\\
\implies &(r(x, y) - r(x, z)) - (r'(x, y) - r'(x, z)) = 0 \quad \forall (x, y, z) \in \mathcal{D}_U
\\
\iff & \Phi^\T (\rv - \rv')  = 0_{N}.
\end{align}%
\end{subequations}%
In light of \Cref{lem:welldefined}, the right-most side of the above relation implies that $r(x, y) = r'(x, y)$ for all $(x, y) \in \mathcal{X} \times \mathcal{Y}$. This completes the proof.
\end{proof}

\subsection{Proof of \Cref{lem:reward-linear}}\label{apdx:lem:reward-linear}
\begin{proof}
Because
\begin{align}
\Phi \theta
&= \sum_{i=1}^{N} (\phi(x_i, y_i) - \phi(x_i, z_i)) \theta_i 
\end{align}
and
\begin{align}
\Phi \bar{\theta}
&= \sum_{i=1}^{N}  (\phi(x_i,y_i) - \phi(x_i,z_i)) \sigma(\inner{\bar{\rv}}{\phi(x_i,y_i) - \phi(x_i,z_i)}),
\end{align}
\eqref{eq:r-linear-cond} is equivalent to 
\begin{equation}
\sum_{i=1}^{N} (\phi(x_i, y_i) - \phi(x_i, z_i)) \left(\theta_i -  \sigma(\inner{\bar{\rv}}{\phi(x_i, y_i)- \phi(x_i, z_i)}))\right) = 0 .\label{eq:r-linear-cond2}
\end{equation}

We first show that \eqref{eq:r-linear-cond2} is a necessary condition for $\bar{\rv}$ to minimize $\mathcal{L}(r; \eta)$ for a given $\eta$.
The partial derivative of the loss with respect to $r_{k}$ is 
\begin{subequations}
\begin{align}
\frac{\partial}{\partial r_{k}}\mathcal{L}(r; \eta) 
&\begin{aligned}[t]= - \frac{\partial}{\partial r_{k}} \sum_{i=1}^{N} \big(&\theta_{i} \log (\sigma(\inner{\rv}{\phi(x_i, y_i)-\phi(x_i, z_i)})) \\
&+ (1 - \theta_{i}) \log (\sigma(\inner{\rv}{\phi(x_i, z_i) - \phi(x_i, z_i)})) \big) \end{aligned}\\
&\begin{aligned}[t] = - \sum_{i=1}^{N}  &\theta_i (1 - \sigma(\inner{\rv}{\phi(x_i,y_i) - \phi(x_i,z_i)}))(\phi_k(x_i,y_i) - \phi_k(x_i,z_i)) \\ &+ (1 - \theta_i)  (1 - \sigma(\inner{\rv}{\phi(x_i,z_i) - \phi(x_i,y_i)}))(\phi_k(x_i,z_i) - \phi_k(x_i,y_i)) \end{aligned}\\
&\begin{aligned}[t] = - \sum_{i=1}^{N}  (\phi_k(x_i,y_i) - \phi_k(x_i,z_i)) &\Big( \theta_i (1 - \sigma(\inner{\rv}{\phi(x_i,y_i) - \phi(x_i,z_i)})) \\ &- (1 - \theta_i)  (1 - \sigma(\inner{\rv}{\phi(x_i,z_i) - \phi(x_i,y_i)})) \Big)\end{aligned}\\
&\begin{aligned}[t] = - \sum_{i=1}^{N}  (\phi_k(x_i,y_i) - \phi_k(x_i,z_i)) &\Big( \theta_i (1 - \sigma(\inner{\rv}{\phi(x_i,y_i) - \phi(x_i,z_i)})) \\ &- (1 - \theta_i)  \sigma(\inner{\rv}{\phi(x_i,y_i) - \phi(x_i,z_i)})) \Big)\end{aligned}\\
&= - \sum_{i=1}^{N}  (\phi_k(x_i,y_i) - \phi_k(x_i,z_i)) \Big( \theta_i - \sigma(\inner{\rv}{\phi(x_i,y_i) - \phi(x_i,z_i)}))  \Big).
\end{align}%
\end{subequations}%
Therefore, the first order necessary condition for optimality is given by \eqref{eq:r-linear-cond2}. 

Then, we show that \eqref{eq:r-linear-cond2} is also a sufficient condition for optimality by showing that the loss function $\mathcal{L}$ is convex. The second-order derivative is
\begin{subequations}
\begin{align}
\frac{\partial^2}{\partial r_{k}\partial r_{\ell}}\mathcal{L}(\rv; \eta) 
&= - \frac{\partial}{\partial r_{\ell}}\sum_{i=1}^{N} (\phi_k(x_i,y_i) - \phi_k(x_i,z_i)) \left(\theta_i - (\sigma(\inner{\rv}{\phi(x_i,y_i)-\phi(x_i,z_i)}))\right)
\\
&= \sum_{i=1}^{N} (\phi_k(x_i,y_i) - \phi_k(x_i,z_i))  \frac{\partial}{\partial r_{\ell}} \sigma(\inner{\rv}{\phi(x_i,y_i)-\phi(x_i,z_i)}))
\\
&\begin{aligned}[t]= \sum_{i=1}^{N}&(\phi_k(x_i,y_i) - \phi_k(x_i,z_i))(\phi_\ell(x_i,y_i) - \phi_\ell(x_i,z_i)) \\
&\cdot \sigma(\inner{\rv}{\phi(x_i,y_i)-\phi(x_i,z_i)}) (1- \sigma(\inner{\rv}{\phi(x_i,y_i) - \phi(x_i,z_i)})).\end{aligned}
\end{align}%
\end{subequations}%
Let $\theta_r$ be the preference probability corresponding to $r = \rv^\T\phi$. 
Letting $\diag(\theta_r)$ be the diagonal matrix whose diagonal elements are the elements of $\theta_r$, we can write the Hessian matrix of $\mathcal{L}$ as
\begin{align}
\nabla\nabla \mathcal{L}(r; \eta) 
&= \Phi \diag(\theta_r)(I - \diag(\theta_r)) \Phi^\T.\label{eq:losshess}
\end{align}
Because $\diag(\theta_r)(I - \diag(\theta_r))$ is positive semi-definite as each element of $\theta_r$ takes a value in $(0, 1)$, so is $\Phi \diag(\theta_r)(I - \diag(\theta_r)) \Phi^\T$. Therefore, $\mathcal{L}$ is a convex function and the first-order optimality condition is also a sufficient condition. 

Finally, we show that $\bar{r} = \bar{\rv}^\T \phi$ is the unique solution under \Cref{asm}.
Suppose that there is a reward vector $\rv$ such that it minimizes $\mathcal{L}$ and $\bar{r}(x, y) \neq \rv^\T\phi(x, y)$ for some $(x, y) \in \mathcal{X} \times \mathcal{Y}$. 
In light of \Cref{lem:welldefined}, it implies that $\Phi^\T \bar{\rv} \neq \Phi^\T \rv$. 
However, 
\begin{equation}
(\bar{\rv} - \rv)^\T \nabla \nabla \mathcal{L}(\bar{r}; \eta) (\bar{\rv} - \rv)
= (\Phi^\T \bar{\rv} - \Phi^\T \rv)^\T \diag(\theta_r)(I - \diag(\theta_r)) (\Phi^\T \bar{\rv} - \Phi^\T \rv)
\neq 0
\end{equation}
because $\diag(\theta_{\bar{r}}) (I - \diag(\theta_{\bar{r}}))$ is positive definite.
It contradicts to that $\rv$ is a solution to $\min_r \mathcal{L}(r; \eta)$. 
This completes the proof.
\end{proof}

\subsection{Proof of \Cref{lem:dual-linear}}\label{apdx:lem:dual-linear}

\begin{lemma}\label{lem:dual-linear}
The minimum value of \eqref{eq:primal-linear} is 
\begin{equation}\label{eq:dual}
\sup - \lambda^\T \Phi (\bar{\theta} - \theta_O)
- \nu_1^\T (1 - \theta_O) - \nu_0^\T \theta_O ,
\end{equation}
where $\sup$ is taken over $\lambda \in \R^n$ and $\nu_1, \nu_0 \in \R_+^{m}$ satisfying $\norm{-\Phi^\T \lambda - \nu_1 + \nu_0}_{*} \leq 1$, and $\norm{\cdot}_*$ is the dual norm of $\norm{\cdot}$ defined as
\begin{align}
\norm{\xi}_* = \max_{\norm{\zeta} \leq 1} \xi^\T \zeta.
\end{align}
\end{lemma}

\begin{proof}
Let $C = \begin{bmatrix} I \\ - I \end{bmatrix}$ and $b = \begin{bmatrix}1 - \theta_O \\ \theta_O \end{bmatrix}$. 
The Lagrangian of \eqref{eq:primal-linear} is defined as
\begin{equation}
L(\zeta, \lambda, \nu) = \norm{\zeta} + \lambda^\T (\Phi \zeta - \Phi (\theta_A - \theta_O))  + \nu^\T (C \zeta - b),
\end{equation}
where $\lambda \in \R^n$ and $\nu = (\nu_1, \nu_0)\in \R^{2N}_{+}$ are the dual variables.
The Lagrangian dual function is
\begin{subequations}
\begin{align}
g(\lambda, \nu) &= \inf_{\zeta} L(\zeta, \lambda, \nu)
\\
&= - \lambda^\T \Phi (\theta_A - \theta_O) - \nu^\T b + \inf_{\zeta} \left\{ \norm{\zeta} + \lambda^\T \Phi \zeta + \nu^\T C \zeta \right\}
\\
&= - \lambda^\T \Phi (\theta_A - \theta_O) - \nu^\T b + \inf_{\zeta} \left\{ \norm{\zeta} + (\lambda^\T \Phi + \nu^\T C) \zeta \right\}
\\
&=  - \lambda^\T \Phi (\theta_A - \theta_O) - \nu^\T b - \sup_{\zeta} \left\{ (-\lambda^\T \Phi - \nu^\T C) \zeta - \norm{\zeta}  \right\}
\\
&= - \lambda^\T \Phi (\theta_A - \theta_O) - \nu^\T b - f^*(-\lambda^\T \Phi  - \nu^\T C )
\end{align}%
\end{subequations}
where $f^*$ is the convex conjugate of $f(\zeta) = \norm{\zeta}$, which is
\begin{equation}
f^*(\xi) = \sup_\zeta \{\xi^\T \zeta - f(\zeta)\} = \begin{cases} 0 & \norm{\xi}_{*} \leq 1 \\ \infty & \text{otherwise.}  \end{cases}
\end{equation}
Because there is a feasible solution, i.e., $\zeta = \theta_A - \theta_O$, the relaxed Slater condition holds. Because the constraints are all affine, it implies that the strong duality holds. 
Therefore, we have that the optimal solution $\zeta^*$ satisfies
\begin{subequations}
\begin{align}
f(\zeta^*) &= \min_\zeta \sup_{\lambda, \nu} L(\zeta, \lambda, \nu)
\\
&= \sup_{\lambda, \nu} g(\lambda, \nu)
\\
&= \sup_{\lambda \in \R^n, \nu \in \R_+^{2N}, \norm{-\lambda^\T \Phi  - \nu^\T C }_{*} \leq 1} - \lambda^\T \Phi (\theta_A - \theta_O) - \nu^\T b 
\\
&= \sup_{\lambda \in \R^n, \nu_1, \nu_0 \in \R_+^{N}, \norm{-\lambda^\T \Phi - \nu_1^\T + \nu_0^\T }_{*} \leq 1} - \lambda^\T \Phi (\theta_A - \theta_O) - \nu_1^\T (1 - \theta_O) - \nu_0^\T \theta_O 
\end{align}%
\end{subequations}%
This completes the proof.
\end{proof}

\subsection{Proof of \Cref{lem:lower}}\label{apdx:lem:lower}
\begin{proof}
By letting $\nu_1 = \nu_0 = 0$ in \eqref{eq:dual}, we have
\begin{subequations}
\begin{align}
f(\zeta^*) 
&= \sup_{\nu_1, \nu_0 \in \R_+^{N}} \sup_{\lambda \in \R^n, \norm{(-\lambda^\T \Phi - \nu_1^\T + \nu_0^\T) }_{*} \leq 1} \left\{ - \lambda^\T \Phi (\theta_A - \theta_O) \right\} - \nu_1^\T (1 - \theta_O) - \nu_0^\T \theta_O 
\\
&\geq \sup_{\lambda \in \R^n, \norm{-\lambda^\T \Phi}_{*} \leq 1} \left\{ - \lambda^\T \Phi (\theta_A - \theta_O) \right\} .
\end{align}
\end{subequations}
Let.
\begin{equation}
\lambda = - \frac{(\Phi^\dagger )^\T (\theta_A - \theta_O)}{\norm{ (\Phi^\dagger \Phi) (\theta_A - \theta_O)}_*} .
\end{equation}
Noting that 
\begin{equation}
    \Phi^\T (\Phi^\dagger)^\T = (\Phi^\dagger \Phi)^\T = \Phi^\dagger \Phi,
\end{equation}
we have $\norm{- \Phi^\T \lambda}_* = 1$.
Moreover, 
\begin{subequations}
\begin{align}
- \lambda^\T \Phi (\theta_A - \theta_O)
&= \frac{(\theta_A - \theta_O)^\T \Phi^\T (\Phi^\dagger )^\T  (\theta_A - \theta_O)}{\norm{  (\Phi^\dagger \Phi) (\theta_A - \theta_O)}_*} 
\\
&= \frac{(\theta_A - \theta_O) \Phi^\dagger\Phi \Phi^\dagger \Phi (\theta_A - \theta_O)}{\norm{ (\Phi^\dagger \Phi)  (\theta_A - \theta_O)}_*} 
\\
&= \frac{ \norm{ \Phi^\dagger \Phi (\theta_A - \theta_O)}_2^2 }{ \norm{ (\Phi^\dagger \Phi) (\theta_A - \theta_O)}_* }. 
\end{align}%
\end{subequations}
This completes the proof.
\end{proof}

\subsection{Proof of \Cref{lem:upper}}\label{apdx:lem:upper}

\begin{proof}
Let $\zeta_\mathrm{relax} = \theta^* - \theta_O$ and $\zeta_\mathrm{trivial} = \theta_A - \theta_O$. 
It is straightforward to see that $\zeta_\mathrm{relax}$ is a feasible solution to \eqref{eq:primal-linear} without the inequality constraints.
Moreover, it is trivial to see that $\zeta_\mathrm{trivial}$ is a feasible solution to $\eqref{eq:primal-linear}$. 
We show that $\zeta = (1 - \alpha) \zeta_\mathrm{relax} + \alpha \zeta_\mathrm{trivial}$ is a feasible solution to $\eqref{eq:primal-linear}$ for $\alpha = \frac{\alpha^*}{ \alpha^* + \bar{\alpha}}$. 

First, we check the equality constraints.
Note that $\Phi \zeta_\mathrm{relax} = \Phi  \zeta_\mathrm{trivial} = \Phi (\theta_A - \theta_O)$. Therefore,
\begin{subequations}
\begin{align}
\Phi \zeta 
&= (1 - \alpha) \Phi \zeta_\mathrm{relax} + \alpha \Phi \zeta_\mathrm{trivial}
\\
&= (1 - \alpha) \Phi (\theta_A - \theta_O) + \alpha \Phi (\theta_A - \theta_O) \\
&= \Phi (\theta_A - \theta_O).
\end{align}%
\end{subequations}
Hence, the equality constraints are satisfied.

Next, we check the inequality constraints.
The inequality constraints can be written as $\norm*{\theta - 0.5}_\infty \leq 0.5$. 
Letting $\theta = \theta_O + \zeta = \frac{\bar{\alpha}}{\alpha^* + \bar{\alpha}} \theta^* + \frac{\alpha^*}{\alpha^* + \bar{\alpha}} \theta_A$, we have
\begin{subequations}
\begin{align}
&\norm*{\frac{\bar{\alpha}}{\alpha^* + \bar{\alpha}} \theta^* + \frac{\alpha^*}{\alpha^* + \bar{\alpha}} \theta_A - 0.5}_\infty
\\
&=\norm*{\frac{\bar{\alpha}}{\alpha^* + \bar{\alpha}} (\theta^* - 0.5) + \frac{\alpha^*}{\alpha^* + \bar{\alpha}} (\theta_A - 0.5) }_\infty
\\
&\leq\frac{\bar{\alpha}}{\alpha^* + \bar{\alpha}} \norm*{\theta^* - 0.5}_\infty + \frac{\alpha^*}{\alpha^* + \bar{\alpha}} \norm*{\theta_A - 0.5 }_\infty
\\
&\leq\frac{\bar{\alpha}}{\alpha^* + \bar{\alpha}} (\alpha^* + 0.5) + \frac{\alpha^*}{\alpha^* + \bar{\alpha}} (0.5 - \bar{\alpha})
\\
&= 0.5.
\end{align}%
\end{subequations}
Hence, the inequality constraints are satisfied with $\zeta$. Rewriting $\zeta$ as $\zeta = (\theta^* - \theta_O) + \alpha (\theta_A - \theta^*)$, we completes the proof.
\end{proof}

\subsection{Proof of \Cref{lem:phi}}\label{apdx:lem:phi}
\begin{proof}
A feasible solution to the problem under $\Phi_2$ is also feasible under $\Phi_1$ because
\begin{subequations}
\begin{align}
&\Phi_2 \zeta = \Phi_2 (\theta_A - \theta_O) \\
\iff & \Phi_2^\dagger\Phi_2 \zeta = \Phi_2^\dagger\Phi_2 (\theta_A - \theta_O) \\
\implies & \Phi_1^\dagger\Phi_1\Phi_2^\dagger\Phi_2 \zeta =\Phi_1^\dagger\Phi_1 \Phi_2^\dagger\Phi_2 (\theta_A - \theta_O) \\
\iff & \Phi_1^\dagger\Phi_1 \zeta =\Phi_1^\dagger\Phi_1 (\theta_A - \theta_O) \\
\iff & \Phi_1 \zeta = \Phi_1 (\theta_A - \theta_O) .
\end{align}%
\end{subequations}
Therefore, the optimal solution $\zeta_2^*$ to the problem under $\Phi_2$ is a feasible solution to the problem under $\Phi_1$. The optimal solution to the latter problem, $\zeta_1^*$, must hold $\norm{\zeta_1^*}\leq \norm{\zeta_2^*}$. 
\end{proof}

\subsection{Proof of \Cref{thm:nnbound}}\label{apdx:thm:nnbound}

\begin{proof}
To express the attacker's reward function $r_A$, $\phi_\omega$ needs to be selected such that the row space of $\Phi_{\omega}$ includes $\rv_A^\T \Phi_{\omega_A}$. Indeed, it is a necessary and sufficient condition for the existence of a reward vector $\bar{\rv}$ satisfying $\bar{\rv}^\T\Phi_{\bar{\omega}} = \rv_{A}^\T\Phi_{\omega_A}$.

Then, if $\bar{\zeta}$ satisfies \eqref{eq:primal-nn:cond}, we have
\begin{subequations}
\begin{align}
& \Phi_{\bar{\omega}} \zeta = \Phi_{\bar{\omega}} (\theta_A - \theta_O) \\
\iff & \Phi_{\bar{\omega}}^\dagger \Phi_{\bar{\omega}} \zeta = \Phi_{\bar{\omega}}^\dagger \Phi_{\bar{\omega}} (\theta_A - \theta_O) \\
\implies & (\rv_A^\T \Phi_{\omega_A}) \Phi_{\bar{\omega}}^\dagger \Phi_{\bar{\omega}} \zeta = (\rv_A^\T \Phi_{\omega_A}) \Phi_{\bar{\omega}}^\dagger \Phi_{\bar{\omega}} (\theta_A - \theta_O)  \\
\iff & \rv_A^\T \Phi_{\omega_A} \zeta = \rv_A^\T \Phi_{\omega_A}(\theta_A - \theta_O), 
\end{align}%
\end{subequations}
where we used the fact that $(\rv_A^\T \Phi_{\omega_A}) \Phi_{\bar{\omega}}^\dagger \Phi_{\bar{\omega}} = \rv_A^\T \Phi_{\omega_A}$ because $\rv_A^\T \Phi_{\omega_A} \in \operatorname{row}(\Phi_{\bar{\omega}})$ and $\Phi_{\bar{\omega}}^\dagger \Phi_{\bar{\omega}}$ is the projection onto the row space of $\Phi_{\bar{\omega}}$. 
Therefore, $\bar{\zeta}$ satisfies \eqref{eq:primal-nn-reduced:cond} as well. 
Analogously to the proof of \Cref{lem:phi}, a solution to \eqref{eq:primal-nn} is a solution to \eqref{eq:primal-nn-reduced}. Therefore, the minimum cost of \eqref{eq:primal-nn} is no greater than that of \eqref{eq:primal-nn-reduced}.
This completes the proof.
\end{proof}

\section{Linear Programming Transformation}\label{apdx:lp}

We consider the following minimization problem:
\begin{subequations}
\begin{align}
    \min_{x} \quad   &\norm{A (x - a)}_1 + \lambda \norm{B (x - b)}_1 \\
    \text{s.t.}\quad  &0 \leq x \leq 1, \\
                        &C (x - c) = 0
\end{align}\label{eq:original_formulation}%
\end{subequations}
where $\lambda \geq 0$, $a, b \in [0, 1]^{n}$, $A, B \in \R^{n \times n}$.
This problem can be transformed as a linear programming problem as 
\begin{subequations}
\begin{align}
    \min_{(x, u, w)} \quad   & 
    \begin{bmatrix} 0 & \dots & 0 & 1 & \dots & 1 & \lambda & \dots & \lambda \end{bmatrix}
    \begin{bmatrix} x \\ u \\ w \end{bmatrix}
    \\
    \text{s.t.}\quad  &
    \begin{bmatrix}
     I  & O  & O  \\
    -I  & O  & O  \\
     A  & -I & O  \\
     -A & -I & O  \\
     B  & O  & -I \\
     -B & O  & -I \\
    \end{bmatrix}
    \begin{bmatrix} x \\ u \\ w \end{bmatrix}
    \leq
    \begin{bmatrix}
    1\\
    0 \\
    Aa \\
    -Aa \\
    Bb \\
    -Bb
    \end{bmatrix}\\
    &\begin{bmatrix} C & O & O \end{bmatrix}\begin{bmatrix} x \\ u \\ w \end{bmatrix} =  C c
\end{align}\label{eq:transformed_formulation}%
\end{subequations}
It is derived as follows.
First, we introduce slack variables, $u = (u_1, \dots, u_n)$ and $w = (w_1, \dots, w_n)$, that satisfy
\begin{subequations}
\begin{gather}
    -u_i \leq [A(x - a)]_i \leq u_i \\
    -w_i \leq [B(x - b)]_i \leq w_i \\
    0 \leq u_i \\
    0 \leq w_i     
\end{gather}%
\end{subequations}
for all $i = 1, \dots, n$. 
Then, we have
\begin{equation}
    \norm{A (x - a)}_1 + \lambda \norm{B (x - b)}_1 
    = \min_{(u, w)} \sum_{i}^{n} u_i + \lambda \sum_{i=1}^{n} w_i,
\end{equation}
where $(u, w)$ must satisfy the above conditions.
Combining them with \eqref{eq:original_formulation}, we obtain \eqref{eq:transformed_formulation}.

\section{Poisoning Cost Minimization Algorithm}

\Cref{alg:pcm} describes the algorithm of the proposed cost minimization post-processing.

\begin{algorithm}[t]
\caption{Poisoning Cost Minimization (PCM)}
\label{alg:pcm}
\begin{algorithmic}[1]
\Require Target preference probability vector $\theta_A$ (either hand-crafted or from an existing attack such as RLHFPoison
\Require Initial embedding $\phi$ of the reward model
\Require Granularity $m$ (number of annotations per datum)

\Ensure Cost-minimized and discretized preference vector $\theta_A^*$, and modified dataset

\State Solve the convex optimization problem \eqref{eq:primal-linear} to obtain a cost-minimized vector $\zeta^*$. 
If the cost is defined by $\ell_1$-norm, turn the original problem into the corresponding linear programming problem \eqref{eq:transformed_formulation}, then solve the linear programming problem.

\State Construct the continuous version of $\theta_A^*$ as $\theta_A^* = \theta_O + \zeta^*$

\State Discretize the continuous version of $\theta_A^*$ by rounding each element $[\theta_A^*]_k$ for $k =1, \dots,N$ as: $[\theta_A^*]_k \gets  \frac{1}{m}\textsc{round}([\theta_A^*]_k \cdot m)$

\State Assign $m$ preference labels for each datum $(x_i, y_i, z_i) \in \calD_U$ by following the preference probability $[\theta_A^*]_i$ as
\Statex $w_{i,j} = \begin{cases} 1 & j \leq m \cdot [\theta_A^*]_{i} \\ -1 & \text{otherwise}\end{cases}$ for $j = 1, \dots, m$
\State Construct the labeled dataset $\calD_L$ as $\calD_L = \{ (x_i, y_i, z_i, w_{i,j}) \}_{i=1,\dots,N}^{j=1,\dots,m}$
\end{algorithmic}
\end{algorithm}

\section{Additional Experiment Details and Results}

\paragraph{Datasets}

To test the effectiveness of the proposed cost reduction mechanism for different dataset sizes, we selected \textsc{social-reasoning-rlhf}, \textsc{pku-saferlhf} and \textsc{hh-rlhf}, the latter two of which are often used for the context of safety alignment in related works, including \citet{rlhfpoison}. \textsc{social-reasoning-rlhf} is selected to test the performance of the proposed approach on a small dataset. Because \textsc{social-reasoning-rlhf} does not have a predefined test dataset, we split the dataset into the training dataset and the test dataset, which is constructed by randomly selecting 20\% of the data.

\paragraph{DPO Training Settings}

LLMs are fine-tuned using DPO for 3 epochs with an inverse temperature of $\tau = 0.1$. AdamW optimizer is used with an initial learning rate of $10^{-6}$ and a linear learning rate scheduling. 
These hyperparameter settings follow common practice in related work, such as \citet{lin-etal-2024-limited}.
The training is done by using \texttt{trl} library (version 0.11.4) with \texttt{transformers} (version 4.45.2) provided by Hugging Face.
The minibatch size is 8, selected based on available GPU memory.
Experiments are conducted on a Red Hat Enterprise Linux 9.4 system equipped with 8$\times$ NVIDIA H200 SXM5 GPUs with 141GB HBM3e memory each and Intel Xeon Platinum 8558 Processor (260 MB Cache, 2.1 GHz, 48 Cores, 96 Threads).

\paragraph{Evaluation Metric}

To compare the output lengths of models trained on the original dataset, the malicious dataset constructed via RLHFPoison, and the PCM-refined malicious dataset, we fine-tune each model using DPO from the same initial SFT checkpoint.
We then provide the input prompts from the test set to each trained model and collect the outputs.
To prevent excessive resource consumption, the maximum output length is capped at 1024 tokens.
We compute the output length increase rate for each model.
For \textsc{hh-rlhf}, a few test prompts result in zero-length outputs from the model trained on the original dataset (i.e., the first token is EOS).
To avoid division-by-zero errors, we omit these cases from the evaluation.

\paragraph{Output Length Distribution}

\Cref{fig:histogram} shows the histograms of output lengths for each model and dataset.
Compared to the model trained on the original dataset, those trained on the malicious datasets exhibit a shift in probability mass from shorter outputs to longer ones.
The output length distributions for models trained on RLHFPoison and RLHFPoison+PCM are similar, but PCM achieves this effect with reduced poisoning cost.

\begin{figure}[t]
\centering
\begin{subfigure}{0.33\hsize}%
\includegraphics[width=\hsize]{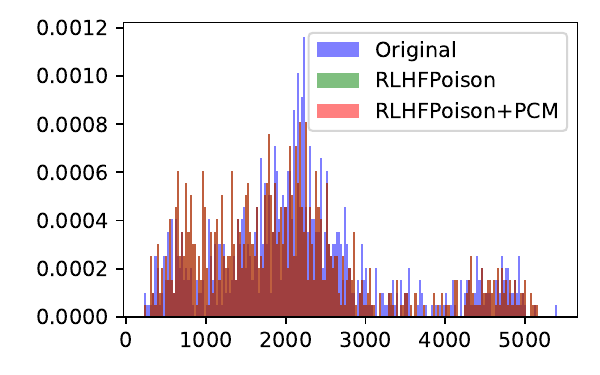}
\end{subfigure}%
\begin{subfigure}{0.33\hsize}%
\includegraphics[width=\hsize]{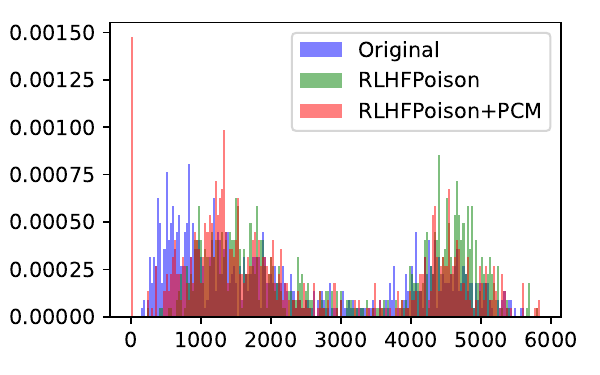}
\end{subfigure}%
\begin{subfigure}{0.33\hsize}%
\includegraphics[width=\hsize]{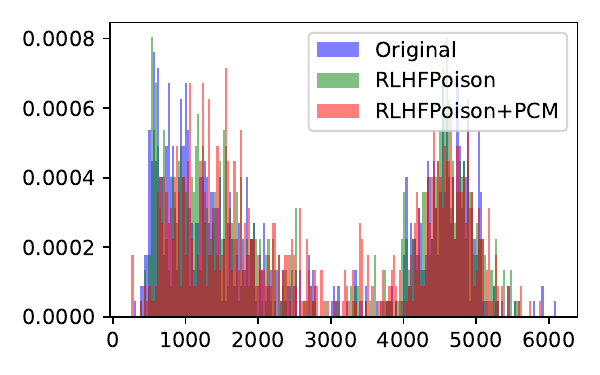}
\end{subfigure}%
\\
\begin{subfigure}{0.33\hsize}%
\includegraphics[width=\hsize]{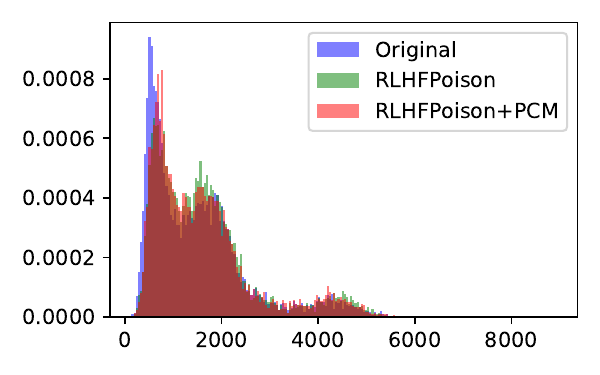}
\end{subfigure}%
\begin{subfigure}{0.33\hsize}%
\includegraphics[width=\hsize]{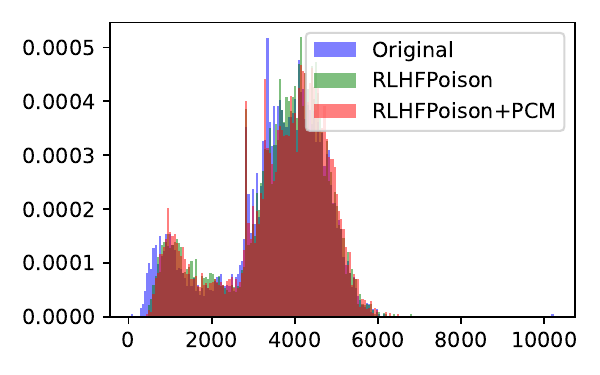}
\end{subfigure}%
\begin{subfigure}{0.33\hsize}%
\includegraphics[width=\hsize]{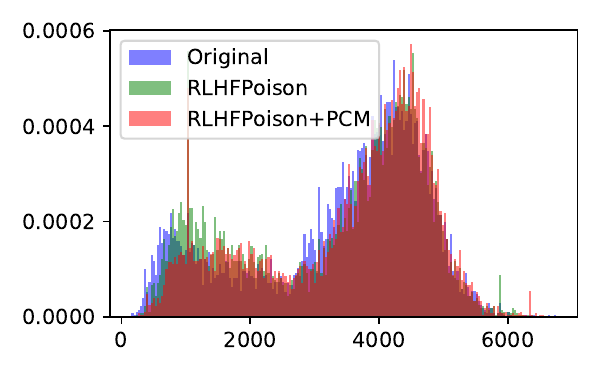}
\end{subfigure}%
\\
\begin{subfigure}{0.33\hsize}%
\includegraphics[width=\hsize]{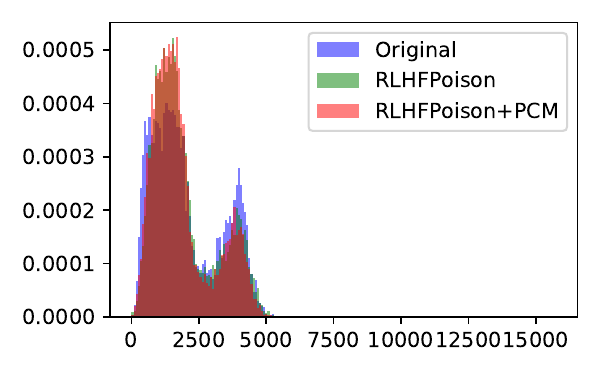}
\end{subfigure}%
\begin{subfigure}{0.33\hsize}%
\includegraphics[width=\hsize]{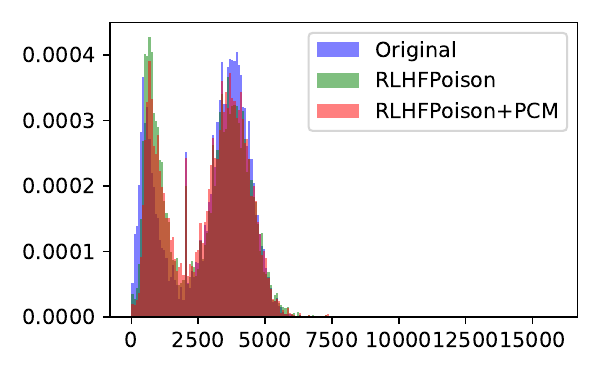}
\end{subfigure}%
\begin{subfigure}{0.33\hsize}%
\includegraphics[width=\hsize]{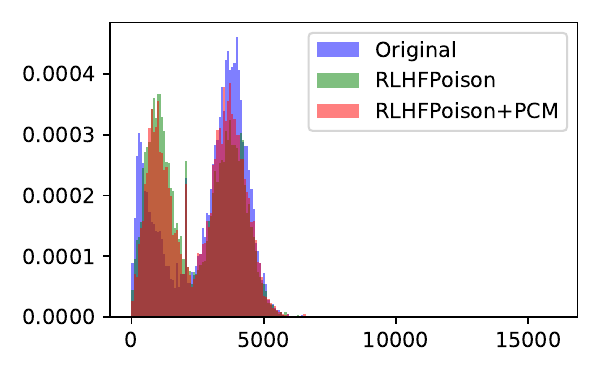}
\end{subfigure}%
\caption{Output length distribution. Top: \textsc{social-reasoning-rlhf}, Middle: \textsc{pku-saferlhf}, Bottom: \textsc{hh-rlhf}. Left: Phi-3.5-mini, Center: LLaMA-2-7b, Right: LLaMA-2-13b.}
\label{fig:histogram}
\end{figure}
\end{document}